\documentclass{article}

\IfFileExists{arxiv.sty}{%
  \usepackage{arxiv}%
}{%
  \usepackage[margin=1in]{geometry}%
}

\usepackage[utf8]{inputenc}
\usepackage[T1]{fontenc}
\usepackage{hyperref}
\usepackage{url}
\usepackage{booktabs}
\usepackage{amsmath,amssymb,amsfonts,mathtools}
\usepackage{amsthm}
\usepackage{nicefrac}
\usepackage{microtype}
\usepackage{graphicx}
\graphicspath{{./images/}}
\usepackage{bbm}
\usepackage{array}
\usepackage[ruled,vlined]{algorithm2e}
\usepackage{color}
\usepackage{colortbl}
\usepackage{enumitem}

\usepackage{booktabs}
\usepackage{multirow}

\hypersetup{colorlinks=true,linkcolor=blue,citecolor=blue,urlcolor=blue}

\newcolumntype{P}[1]{>{\centering\arraybackslash}p{#1}}
\newcolumntype{R}[1]{>{\raggedleft\arraybackslash}p{#1}}

\definecolor{c_diag}{rgb}{0.85,0.89,0.953}
\definecolor{t_diag}{rgb}{0,0,0}
\definecolor{c_ndiag}{rgb}{.95,0.95,0.95}
\definecolor{c_hlight}{rgb}{.886,0.9411,0.851}
\definecolor{c_hdark}{rgb}{.663,0.82,0.557}
\definecolor{c_rest}{rgb}{.984,0.898,0.839}

\newtheorem{definition}{Definition}
\newtheorem{assumption}{Assumption}
\newtheorem{lemma}{Lemma}
\newtheorem{proposition}{Proposition}
\newtheorem{theorem}{Theorem}
\newtheorem{corollary}{Corollary}

\DeclareMathOperator{\BC}{BC}
\DeclareMathOperator{\BD}{BD}

\newcommand{\Fcal}{\mathcal{F}}
\newcommand{\Lcal}{\mathcal{L}}
\newcommand{\Rcal}{\mathcal{R}}

\newcommand{\Ecal}{\mathcal{E}}

\newcommand{\Prob}{\mathbb{P}}

\newcommand{\R}{\mathbb{R}}

\title{A Calibrated Stopping Rule for Feature Subset Selection from Feature Rankings}
\title{When to Stop Selecting Features from Feature Rankings: A Residual-Overlap Stopping Rule}
\title{When to Truncate a Feature Ranking: A Residual-Overlap Stopping Rule for Subset Selection}

\author{
 Jes\'us S. Aguilar--Ruiz \\
  School of Engineering\\ Pablo de Olavide University\\ 
  ES-41013 Seville, Spain \\
  \texttt{aguilar@upo.es} \\
 }

\begin{document}
\maketitle

\begin{abstract}
Feature rankings are widely used in supervised feature selection because they are simple, scalable and easy to interpret. Variables are first ranked by a relevance score, and a subset is then obtained by retaining the top-ranked variables. Although the first stage has been extensively studied, the second is often governed by an arbitrary cardinality, an empirical threshold or cross-validation, without a direct interpretation. This raises a basic question: given a feature ranking, when is there enough accumulated class-separation evidence to stop selecting features?

To answer this question, this paper develops a distributional framework for transforming supervised feature rankings into class-independent subsets through an explicit risk-calibrated stopping rule. For each variable and each pair of classes, marginal separation is measured by the Bhattacharyya coefficient between the corresponding class-conditional distributions. Under the product marginal model induced by a selected subset, these variable-wise overlaps factorise exactly, yielding an additive log-overlap representation of accumulated marginal evidence. The proposed method selects a single global subset shared by all classes by retaining the shortest prefix of a ranking whose residual product overlap falls below a prescribed threshold for every relevant class contrast. We derive binary and multiclass Bayes-risk bounds for the labelled product marginal problem, obtain prior-dependent and prior-free calibrations of the residual-overlap threshold from a target all-pairs risk level, discuss plug-in estimation of the marginal overlaps and explicitly delimit the model-based calibration gap induced by dependence among selected variables.

An empirical comparison on high-dimensional genomic datasets illustrates that the rule can reduce tens of thousands of variables to a few dozen while maintaining predictive performance statistically comparable to the all-features baseline. As the stopping rule only requires one-dimensional marginal overlap estimates and scans a precomputed ranking, it is well suited to very high-dimensional settings where exhaustive subset search is infeasible and interpretable truncation of feature rankings is essential.

\end{abstract}


\section{Introduction}\label{sec:introduction}

High-dimensional classification problems are now routine in bioinformatics, medical diagnosis, and other data-intensive scientific domains. In such problems the number of measured variables can be much larger than the number of observations. Learning directly from the full variable set may increase variance, reduce interpretability, and make the resulting model difficult to audit. Feature selection addresses these issues by replacing the original variable set with a smaller subset that is expected to preserve the discriminative information relevant to the supervised task \cite{Blum1997,Kohavi1997,Dash1997,Guyon2003,LiuYu2005,Saeys2007}.

A substantial part of feature-selection practice is based on rankings. Each variable is assigned a relevance score, for example by a statistical test \cite{Student1908}, a mutual-information criterion \cite{Peng2005}, a margin-based estimator such as ReliefF \cite{Robnik2003}, the magnitude or sparsity pattern of model coefficients as in the lasso \cite{Tibshirani1996}, impurity reduction as in random forests \cite{Breiman2001}, or a distributional separation index \cite{Xuan2006}. Variables are then ordered from most to least relevant. This formulation is attractive because it is simple, scalable and interpretable. Nevertheless, a ranking is not yet a subset. A ranked list must still be cut at some point, and this stopping decision is often less theoretically justified than the score used to construct the ranking. In applied work, the cut is frequently chosen by retaining a fixed number of variables, by taking a fixed percentage of the list, by applying wrapper strategies to the ranking \cite{Ruiz2006}, or by cross-validation. Such decisions may be operationally convenient, but they do not necessarily explain what the selected threshold means in probabilistic or statistical terms.

The multiclass setting adds a further distinction between the information used to assess variables and the representation ultimately returned by the selector. A conventional class-independent method returns a single subset of variables, shared by all classes; this is the natural output when the aim is to provide a common input space for downstream classifiers. Class-specific feature selection, by contrast, allows different classes to be represented by different subsets \cite{Oh1999,Pineda2011,AguilarRuiz2025}. 

In supervised feature selection, the assessment of a variable is usually based on class information: a feature is considered relevant because its distribution, association, or predictive behaviour changes across classes. This use of class-conditional evidence does not, by itself, make a method class-specific. The distinction depends on the form of the subset returned by the selector. A class-specific method returns $k$ different subsets for $k$ different classes, for example $S_1,\ldots,S_k$. A class-independent method, in contrast, returns one common subset $T\subseteq\mathcal F$, which is used as the shared representation for all $k$ classes.

The method proposed here is class-independent. Pairwise class-conditional overlaps are used to monitor whether all class contrasts have accumulated sufficient marginal separation under the product marginal model, but the selected representation remains one common subset shared by all classes.

This paper addresses the following question: given a supervised feature ranking, where should selection stop if the retained subset is to have a clear statistical interpretation? We focus on rankings whose variables can be evaluated through the separation of their class-conditional marginal distributions. This viewpoint is natural because supervised discrimination depends on how the distribution of a variable changes across classes. Distributional separation scores, including Bhattacharyya- and Wasserstein-type criteria, can produce informative rankings \cite{Xuan2006,CaoZhang2024}. However, they do not by themselves answer the separate question of where such rankings should be truncated.

We introduce a calibrated residual-overlap stopping principle. For each ranked variable and each relevant class contrast, the method estimates the class-conditional marginal overlap and accumulates these overlaps along the ranking under the product marginal model. Selection stops at the shortest prefix whose residual product overlap is below a prescribed threshold for all class contrasts simultaneously. The resulting threshold has a model-based Bayes-risk-calibration interpretation for an explicitly defined labelled product marginal problem. The main contribution is a calibrated truncation rule that converts an existing supervised ranking into a single class-independent subset of original variables, with an explicit residual-overlap interpretation under the labelled product marginal model.

\section{Related Work}\label{sec:related}

\subsection{Classical feature selection, ranking and subset extraction}

Feature selection has a long history in machine learning, statistics and pattern recognition. Standard taxonomies distinguish filter, wrapper and embedded methods \cite{Blum1997,Kohavi1997,Dash1997,Guyon2003,LiuYu2005, Chandrashekar2014}. Filter methods evaluate variables without committing to a particular downstream learner; wrapper methods search for subsets through repeated predictive evaluation; and embedded methods perform selection as part of the optimisation problem used to fit the classifier or regressor. In very high-dimensional domains, especially in bioinformatics \cite{Wellinger2022}, filtering and ranking methods remain attractive because they scale well and produce interpretable lists of candidate variables \cite{Saeys2007}.

Many classical filters assign a relevance score to each variable, either individually or through a greedy relevance--redundancy trade-off. Relief-type methods estimate feature relevance by comparing neighbouring examples from the same and different classes \cite{KiraRendell1992,Robnik2003}. Fast correlation-based filters remove irrelevant and redundant variables efficiently in high-dimensional spaces \cite{YuLiu2003}. Mutual-information filters, including minimum-redundancy maximum-relevance, seek variables that are strongly associated with the response while avoiding repeated information among selected variables \cite{Peng2005,Brown2012}. These methods are directly relevant to the present paper because they expose a distinction that is often hidden in feature-selection practice: a variable may be individually relevant, but its usefulness after other variables have been retained depends on the information already accumulated by the subset.

Embedded sparse models form a second major tradition. The lasso (least absolute shrinkage and selection operator) \cite{Tibshirani1996}, the elastic net \cite{ZouHastie2005}, sparse support vector machines and related regularised learners perform feature selection through penalty parameters or sparsity-inducing constraints. Stability selection \cite{Meinshausen2010} adds a resampling layer and provides finite-sample control for certain false-selection quantities. These approaches provide principled mechanisms for controlling model complexity, for example through regularisation, sparsity constraints, resampling or validation. However, the final number of selected variables is usually interpreted indirectly: it is the subset obtained at a particular value of a penalty parameter, for a particular predictive model, under a chosen stability threshold, or at the cross-validated optimum. It is therefore not usually attached to a direct distributional or risk-calibrated meaning.

\subsection{Class-independent and class-specific feature selection}

Most feature-selection methods are class-independent: they return a single subset used for all classes. This is appropriate when the selected variables are intended to define a single input space shared by all classes. The same subset is then used to train classifiers, evaluate test samples, report selected variables and perform validation, making the resulting representation common, reproducible and easy to compare across classes. In multiclass problems, however, different variables may be useful for different class contrasts. This observation motivates class-specific feature selection, where relevance is allowed to vary with the class used to solve the multiclass problem.

Early work by Oh, Lee and Suen \cite{Oh1999} studied class separation and the combination of class-dependent features for handwriting recognition, showing that the discriminative value of a feature may vary by class. Pineda-Bautista, Carrasco-Ochoa and Mart\'{i}nez-Trinidad \cite{Pineda2011} proposed a general framework for class-specific feature selection using one-against-all binarisation, so that different subsets may be assigned to different classes. Zhou and Dickerson \cite{ZhouDickerson2014} studied class-dependent feature selection for cancer biomarker discovery. Nardone, Ciaramella and Staiano \cite{Nardone2019} developed a sparse-modelling approach for class-specific feature selection. Das and Pal \cite{DasPal2022} embedded class-specific and rule-specific feature selection in a fuzzy rule-based framework with redundancy control. More recently, Aguilar-Ruiz \cite{AguilarRuiz2025} gave an explicit account of class-specific feature selection for explainability, relating it to one-versus-all and one-versus-each classification schemes and contrasting it with the class-independent paradigm.

The present paper uses pairwise class-conditional information to construct and calibrate a single subset shared by all classes. Class specificity concerns the form of the output representation, not the information used to score or monitor variables. In the proposed approach, all pairwise contrasts are monitored, but the final feature set remains class-independent. 

\subsection{Distributional separation between class-conditional laws}

A natural way to evaluate a supervised feature is to compare its class-conditional distributions. If the distribution of a variable is nearly the same across classes, the variable has little marginal ability to distinguish them; if the distributions are well separated, the variable carries marginal discriminative information. This distributional viewpoint underlies many ranking criteria, including statistical tests, divergence-based measures, kernel two-sample statistics, optimal transport and integral probability metrics.

Classical test-based filters rank variables through evidence of distributional or summary-level differences, as in the $t$-test \cite{Student1908}, the Mann--Whitney test \cite{MannWhitney1947} or the Kolmogorov--Smirnov statistic \cite{Massey1951}. Divergence-based filters replace testing by a numerical measure of discrepancy or affinity, such as Kullback--Leibler divergence \cite{KullbackLeibler1951}, Jensen--Shannon divergence \cite{Lin1991}, Hellinger distance or Bhattacharyya affinity \cite{Bhattacharyya1943}. Kernel methods provide another non-parametric route: maximum mean discrepancy compares distributions through their mean embeddings in a reproducing-kernel Hilbert space \cite{Gretton2012}. Optimal-transport criteria compare distributions geometrically, with Wasserstein distances providing a prominent example \cite{PeyreCuturi2019}. Recent work has used such distances between class-conditional distributions, including Wasserstein-based variants and robustness-motivated extensions \cite{CaoZhang2024,Li2024}. Integral probability metrics offer a common abstraction for several of these approaches, since they compare distributions through the largest difference of expectations over a chosen class of test functions \cite{Mueller1997}.

Bhattacharyya and Hellinger quantities are particularly relevant here because they measure distributional overlap and are directly connected with classical Bayes-error bounds. This connection has motivated feature-selection procedures based on the Bhattacharyya distance \cite{Xuan2006}, stable sparse selection based on the Hellinger distance \cite{Fu2020}, and renewed analyses of Bhattacharyya and Chernoff upper bounds on Bayes error \cite{Nielsen2014}. These works support the broader premise that supervised feature evaluation can be based on distances or affinities between class-conditional laws.

The present work belongs to this distributional family, but addresses a different question. It does not propose a new distributional score to optimise over subsets. Instead, it estimates variable-wise overlaps between class-conditional marginal laws and uses their product-model accumulation to decide where a given ranking should be truncated. The methodological emphasis therefore shifts from choosing the score that defines the ranking to assigning the stopping point a calibrated residual-overlap interpretation.

\subsection{Bhattacharyya coefficient, Hellinger distance and Bayes-error bounds}

The Bhattacharyya coefficient is a classical affinity between probability distributions \cite{Bhattacharyya1943}. Its negative logarithm, the Bhattacharyya distance, has long been used in signal selection, statistical pattern recognition and classification theory \cite{Kailath1967,Fukunaga1990}. The coefficient is closely related to the squared Hellinger distance and to upper bounds on the Bayes error. In binary classification, the elementary inequality between a minimum and a geometric mean yields a Bhattacharyya upper bound on the Bayes risk. Chernoff bounds generalise this mechanism by optimising over a family of affinities. Nielsen \cite{Nielsen2014} reviews and extends Bhattacharyya and Chernoff upper bounds, emphasising their usefulness when the exact Bayes error is analytically intractable.

Bhattacharyya-type quantities have also been used directly for feature selection. Xuan et al. \cite{Xuan2006} proposed a Bhattacharyya-distance-based method for multiclass feature selection, using a recursive algorithm to obtain a dimension-reduction matrix that minimises an upper bound on classification error under normal distributions. Their work is closely related in spirit because it exploits the connection between Bhattacharyya distance and error bounds. The present paper differs in three respects. First, it works with variable-wise marginal overlaps rather than a multivariate Gaussian dimension-reduction matrix. Secondly, it uses the exact factorisation of Bhattacharyya affinity under an explicitly stated product marginal model. Thirdly, it interprets the resulting product overlap as a residual quantity used to decide when a ranked list should stop.

\subsection{Stopping rules, stability and calibration}

A recurring weakness of ranking-based feature selection is that the ranking criterion and the subset size are often justified separately. Cross-validation can tune the number of retained variables, but the resulting cardinality may be unstable when the sample size is small, the feature space is large, or many features have similar scores. Stability selection provides a reproducibility-oriented solution by controlling expected false selections under subsampling \cite{Meinshausen2010}. Redundancy-aware criteria such as mRMR control one source of over-selection by discouraging repeated information \cite{Peng2005}. Graph-based and differentiable methods can incorporate redundancy, manifold structure, sparsity or gates into the learning objective \cite{Roffo2021,Zhang2022,Huang2024,Wang2021,Lindenbaum2021,Masoomi2020}. Nevertheless, these approaches address questions that are distinct from the one considered here.

The contribution of this paper lies between distributional feature ranking and error-bound calibration. The method assumes that class-conditional marginal overlaps can be estimated and uses their product-model accumulation to decide where a ranked list should be truncated. The resulting threshold has a direct interpretation for the labelled product marginal problem, making the underlying model assumption explicit.

\section{Foundations}\label{sec:foundations}

Let $Y\!\in\!\Lcal\!=\!\{c_1,\ldots,c_k\}$ be a class label with class prior probabilities $\pi_i\!=\!\Prob(Y\!=\!c_i), i\!=\!1,\ldots,k$, and let $X\!=\!(X_1,\ldots,X_m)$ be a vector of $m$ variables indexed by $\Fcal=\{1,\ldots,m\}$. In a sample of size $n$, the observed data may be written as $\{(x_r,y_r):r=1,\ldots,n\}$, where $x_r\in\R^m$ and $y_r\in\Lcal$. For $T\subseteq\Fcal$, $X_T$ denotes the subvector of $X$ containing the variables indexed by $T$.

\begin{definition}[Class-independent and class-specific selectors]
A class-independent feature selector returns a single subset $T\subseteq\Fcal$ that is used for all classes. A class-specific selector may instead return a map $S:\Lcal\to 2^{\Fcal}, c_i\mapsto S_i$, or pairwise subsets $T_{ij}\subseteq\Fcal, 1\le i<j\le k$. 
\end{definition}
The method in this paper is class-independent: its principal output is one global subset $T$. Pairwise quantities are used only to measure and control discriminative overlap across class contrasts.

\begin{definition}[Feature ranking]
A feature ranking is a permutation $\mathcal R=(v(1),\ldots,v(m))$ of $\mathcal F$. It may be induced by a score $s:\mathcal F\to\mathbb R$, where larger values indicate greater relevance, so that
\begin{equation*}
s(v(1))\ge s(v(2))\ge \cdots \ge s(v(m)).
\end{equation*}
The ranking may be produced by any scoring mechanism. The method below does not require any particular ranking construction; it requires only a ranking and estimates of the residual overlaps used for stopping.
\end{definition}

For class $c_i$ and variable $v$, let $P_{i,v}$ denote the class-conditional marginal law of $X_v\!\mid\! Y\!=\!c_i$. We assume that, for each $v$, the distributions $P_{1,v},\ldots,P_{k,v}$ are dominated by a common $\sigma$-finite measure $\mu_v$, and write $f_{i,v}$ for the corresponding density or mass function. The unconditional marginal distribution of $X_v$ is the mixture
\begin{equation}
P_v=\sum_{i=1}^k\pi_i P_{i,v}.
\label{eq:marginal_mixture}
\end{equation}

Consequently, an unconditional feature distribution may be multimodal even when each class-conditional marginal distribution is unimodal, for example Gaussian. The proposed criterion does not model the unconditional marginal distribution in \eqref{eq:marginal_mixture}; it measures separation through the overlaps of the class-conditional marginals.

\begin{definition}[Bhattacharyya coefficient]
Let $P$ and $Q$ be probability distributions dominated by a common measure $\mu$, with densities $p$ and $q$. Their Bhattacharyya coefficient is
\begin{equation}
\BC(P,Q)=\int \sqrt{p(x)q(x)}\,d\mu(x).
\label{eq:bc_def}
\end{equation}
The Bhattacharyya distance is
\begin{equation}
\BD(P,Q)=-\log \BC(P,Q),
\label{eq:bd_def}
\end{equation}
with the convention $-\log 0=+\infty$.
\end{definition}

For a variable $v$ and a class pair $i<j$, define
\begin{equation}
\beta_v^{ij}=\BC(P_{i,v},P_{j,v}),
\qquad
w_v^{ij}=-\log \beta_v^{ij}.
\label{eq:beta_weight}
\end{equation}
The quantity $\beta_v^{ij}$ measures marginal overlap for the contrast $(c_i,c_j)$, whereas $w_v^{ij}$ measures additive marginal log-separation. Products of the $\beta_v^{ij}$ coefficients define residual product overlap for selected subsets. If $\beta_v^{ij}=0$, the two class-conditional marginals are mutually singular and the variable alone eliminates product-model overlap for that contrast.

For a subset $T\subseteq\Fcal$, define the product marginal model for class $c_i$ by
\begin{equation}
P_{i,T}^{\otimes}=\bigotimes_{v\in T}P_{i,v}.
\label{eq:product_model}
\end{equation}
This is not necessarily the true joint distribution of $X_T\!\mid\! Y\!=\!c_i$. It is the distribution obtained by keeping the class-conditional marginal law of each selected variable and replacing their joint dependence structure within class $c_i$ by conditional independence. If the true joint class-conditional law is denoted by $P_{i,T}=\mathcal{L}(X_T\mid Y=c_i)$, then generally $P_{i,T}\ne P_{i,T}^{\otimes}$.

The product construction in \eqref{eq:product_model} is the same modelling device that underlies naive Bayes classifiers, one of the most widely used probabilistic baselines in supervised learning. Naive Bayes replaces the unknown joint class-conditional distribution by the product of its marginal class-conditional distributions, thereby accumulating feature-wise evidence through a transparent product rule. Equality in \eqref{eq:product_model} holds under conditional independence of the selected variables within class $c_i$, but the proposed method does not require that equality as an empirical claim. Instead, it adopts the product model as an explicit working model for calibrated marginal evidence accumulation. In this sense, the residual-overlap rule inherits the interpretability of naive Bayes while using the product construction only to define a model-based stopping criterion.

For $i<j$, define the product-model residual overlap of a subset $T$ by
\begin{equation}
B_{ij}(T)=\BC(P_{i,T}^{\otimes},P_{j,T}^{\otimes}).
\label{eq:bij_def}
\end{equation}

\begin{lemma}[Factorisation of product Bhattacharyya overlap]
\label{lem:factorisation}
For any subset $T\subseteq\Fcal$ and any pair $i<j$,
\begin{equation}
B_{ij}(T)=\prod_{v\in T}\beta_v^{ij}.
\label{eq:factorisation}
\end{equation}
Equivalently,
\begin{equation}
-\log B_{ij}(T)=\sum_{v\in T} w_v^{ij}.
\label{eq:additive_overlap}
\end{equation}
\end{lemma}

\begin{proof}
Let $f_{i,v}$ and $f_{j,v}$ be densities of $P_{i,v}$ and $P_{j,v}$ with respect to $\mu_v$, and let $\mu_T=\bigotimes_{v\in T}\mu_v$. Under the product model,
\begin{equation*}
f_{i,T}^{\otimes}(x_T)
=
\prod_{v\in T}f_{i,v}(x_v),
\qquad
f_{j,T}^{\otimes}(x_T)
=
\prod_{v\in T}f_{j,v}(x_v).
\end{equation*}
Therefore,
\begin{align*}
\operatorname{BC}(P_{i,T}^{\otimes},P_{j,T}^{\otimes})
&=
\int
\sqrt{
f_{i,T}^{\otimes}(x_T)
f_{j,T}^{\otimes}(x_T)
}
\,d\mu_T(x_T) \nonumber\\
&=
\int
\prod_{v\in T}
\sqrt{f_{i,v}(x_v)f_{j,v}(x_v)}
\,d\mu_T(x_T).
\end{align*}
The integrand is non-negative and measurable, so Tonelli's theorem \cite{Folland1999} gives
\begin{align*}
\operatorname{BC}(P_{i,T}^{\otimes},P_{j,T}^{\otimes})
&=
\prod_{v\in T}
\int
\sqrt{f_{i,v}(x_v)f_{j,v}(x_v)}
\,d\mu_v(x_v) \nonumber\\
&=
\prod_{v\in T}\beta_v^{ij}.
\end{align*}
Taking negative logarithms gives \eqref{eq:additive_overlap}.
\end{proof}

\subsection{Bayes-risk bounds}

For two classes $c_i$ and $c_j$, consider the binary product-model problem obtained by replacing the joint class-conditional laws on $T$ by $P_{i,T}^{\otimes}$ and $P_{j,T}^{\otimes}$. We use the original class priors $\pi_i$ and $\pi_j$, rather than renormalising them over the pair. The resulting pairwise Bayes-error contribution is
\begin{equation}
R_{ij}^{\star,\otimes}(T)
=
\int
\min\{
\pi_i f_{i,T}^{\otimes}(x),
\pi_j f_{j,T}^{\otimes}(x)
\}
\,d\mu_T(x),
\label{eq:pairwise_bayes_contribution}
\end{equation}
where $f_{i,T}^{\otimes}$ and $f_{j,T}^{\otimes}$ are densities of $P_{i,T}^{\otimes}$ and $P_{j,T}^{\otimes}$ with respect to the product measure $\mu_T$. This quantity is the Bayes risk contribution associated with the binary contrast $(c_i,c_j)$ under the product marginal model. 

\begin{lemma}[Pairwise Bhattacharyya bound]
\label{lem:pairwise_bound}
For any $T\subseteq\Fcal$ and $i<j$,
\begin{equation}
R_{ij}^{\star,\otimes}(T)
\le
\sqrt{\pi_i\pi_j}\,B_{ij}(T).
\label{eq:pairwise_bound}
\end{equation}
\end{lemma}

\begin{proof}
For any non-negative $a$ and $b$, $\min\{a,b\}\le \sqrt{ab}$. Applying this inequality to the integrand in \eqref{eq:pairwise_bayes_contribution} gives
\begin{align*}
R_{ij}^{\star,\otimes}(T)
&\le
\int
\sqrt{
\pi_i f_{i,T}^{\otimes}(x)
\pi_j f_{j,T}^{\otimes}(x)
}
\,d\mu_T(x) \nonumber\\
&=
\sqrt{\pi_i\pi_j}
\int
\sqrt{
f_{i,T}^{\otimes}(x)
f_{j,T}^{\otimes}(x)
}
\,d\mu_T(x) \nonumber\\
&=
\sqrt{\pi_i\pi_j}\,
\operatorname{BC}
\left(
P_{i,T}^{\otimes},
P_{j,T}^{\otimes}
\right).
\end{align*}
By the product-model factorisation of the Bhattacharyya coefficient,
\begin{equation*}
\operatorname{BC}
\left(
P_{i,T}^{\otimes},
P_{j,T}^{\otimes}
\right)
=
B_{ij}(T).
\end{equation*}
Hence
\begin{equation*}
R_{ij}^{\star,\otimes}(T)
\le
\sqrt{\pi_i\pi_j}\,B_{ij}(T),
\end{equation*}
as claimed.
\end{proof}

\section{Method}\label{sec:method}

\subsection{Objective} \label{subsec:objective}

The goal is to select a single class-independent subset $T\subseteq\Fcal$ whose residual overlap is controlled for all relevant class contrasts. For each pair $(c_i,c_j)$, let $\theta_{ij}\in(0,1]$ denote the residual-overlap threshold assigned to that contrast. A subset $T$ is called pairwise calibrated if
\begin{equation}
B_{ij}(T)
\le
\theta_{ij},
\qquad
1\le i<j\le k.
\label{eq:pairwise_calibrated}
\end{equation}
Thus $\theta_{ij}$ specifies the maximum residual product overlap allowed for the pair $(c_i,c_j)$. The general formulation allows a different residual-overlap threshold $\theta_{ij}\in(0,1]$ for each class contrast. This makes it possible to assign stricter tolerances to pairs that are scientifically or operationally more important, or to contrasts that are known to be difficult. For simplicity, and unless stated otherwise, we use the uniform-threshold case $\theta_{ij}=\theta$ for all $i<j$. This default choice gives a single class-independent residual-overlap threshold while retaining the pairwise formulation needed for calibration and diagnostics.

Thus, in the uniform-threshold case, this reduces to
\begin{equation}
B_{ij}(T)
\le
\theta,
\qquad
1\le i<j\le k.
\end{equation}

\subsection{Class-independent calibrated prefix}

Let $\Rcal=(v_{(1)},\ldots,v_{(m)})$ be a ranking. For $q=0,1,\ldots,m$, define the prefix $T_q(\Rcal)=\{v_{(1)},\ldots,v_{(q)}\}$, with $T_0(\Rcal)=\varnothing$. Given a uniform residual-overlap threshold $\theta\in(0,1]$, the class-independent calibrated prefix length is
\begin{equation}
q^\star
=
\min\left\{
q\in\{0,\ldots,m\}:
B_{ij}(T_q(\Rcal))\le \theta
\text{ for all }1\le i<j\le k
\right\},
\label{eq:calibrated_prefix}
\end{equation}
provided the set is non-empty. The selected subset is
\begin{equation}
T^\star=T_{q^\star}(\Rcal).
\label{eq:class_independent_subset}
\end{equation}

At the selected prefix length $q^\star$, the stopping condition is equivalently
\begin{equation}
\sum_{\ell=1}^{q^\star}w_{v_{(\ell)}}^{ij}\ge -\log\theta
\quad
\text{for all }1\le i<j\le k.
\label{eq:log_stopping}
\end{equation}
Thus the method stops once the accumulated log-separation reaches the prescribed accumulated log-separation level for every class contrast.

For $q=0$, the selected prefix is empty, $T_0(\Rcal)=\varnothing$. Since $B_{ij}(T)$ is defined as a product over the features in $T$, the empty-product convention gives $B_{ij}(T_0(\Rcal))=\prod_{v\in\varnothing}\beta_v^{ij}=1$. Thus, before any feature is selected, the residual overlap is maximal; in particular, if $\theta<1$, the stopping rule cannot stop at $q=0$. This is consistent with the calibration principle, since a non-trivial residual-overlap threshold can only be reached after at least one selected feature contributes positive log-separation.

\begin{proposition}[Feasibility of a calibrated prefix] \label{prop:feasibility}
For a fixed ranking $\Rcal=(v_{(1)},\ldots,v_{(m)})$ and a uniform residual-overlap threshold $\theta\in(0,1]$, the calibrated prefix length $q^\star$ in \eqref{eq:calibrated_prefix} is well defined if and only if
\begin{equation}
B_{ij}(T_m(\Rcal))
\le
\theta
\qquad
\text{for all }1\le i<j\le k.
\label{eq:prefix_feasibility}
\end{equation}
Equivalently,
\begin{equation}
\sum_{\ell=1}^{m}w_{v_{(\ell)}}^{ij}
\ge
-\log\theta
\qquad
\text{for all }1\le i<j\le k.
\label{eq:prefix_feasibility_log}
\end{equation}
\end{proposition}

\begin{proof}
The prefixes are nested: $T_0(\Rcal)\subseteq T_1(\Rcal)\subseteq\cdots\subseteq T_m(\Rcal)$. Since every coefficient $\beta_v^{ij}$ lies in $[0,1]$, each residual overlap
$B_{ij}(T_q(\Rcal))$ is non-increasing in $q$. Therefore, if some prefix $T_q(\Rcal)$ satisfies the threshold condition $B_{ij}(T_q(\Rcal))\le\theta$ for all $i<j$, then the full prefix $T_m(\Rcal)$ also satisfies it. Conversely, if $T_m(\Rcal)$ satisfies all pairwise threshold conditions, then the feasible set in \eqref{eq:calibrated_prefix} is non-empty. Since this set is a finite subset of $\{0,\ldots,m\}$, it has a minimum, and hence $q^\star$ is well defined.

The logarithmic formulation follows from $B_{ij}(T_m(\Rcal))=\prod_{\ell=1}^m\beta_{v_{(\ell)}}^{ij}$ and $w_v^{ij}=-\log\beta_v^{ij}$, with the usual extended-real convention when $\beta_v^{ij}=0$.
\end{proof}

Along a fixed ranking, the residual overlaps $B_{ij}(T_q(\Rcal))$ are non-increasing in $q$, since each additional feature contributes a multiplicative factor in $[0,1]$. Therefore, $T_m(\Rcal)$ is the strongest prefix available under the given ranking. If $T_m(\Rcal)$ does not satisfy the threshold condition for some pair $(c_i,c_j)$, no shorter prefix can do so. Failure of feasibility indicates that the prescribed tolerance is unattainable from the available ranked features for at least one class contrast.

When the feasibility condition holds, the stopping rule selects the smallest prefix length at which all residual overlaps fall below the common threshold $\theta$, or equivalently, all accumulated log-separations reach $-\log\theta$. The pairs for which $B_{ij}(T_m(\Rcal))>\theta$ identify the bottleneck contrasts, namely those that remain insufficiently separated even after all ranked features have been included.

\begin{proposition}[Binary cardinality optimality] \label{prop:binary_cardinality}
Assume $k=2$ and suppose that the ranking $\Rcal$ contains all candidate features and is ordered so that
\begin{equation}
w_{v_{(1)}}^{12}
\ge
w_{v_{(2)}}^{12}
\ge
\cdots
\ge
w_{v_{(m)}}^{12}.
\end{equation}
If the feasible set in \eqref{eq:calibrated_prefix} is non-empty, then $T^\star$ has minimum cardinality among all subsets $T\subseteq\Fcal$ satisfying
\begin{equation}
B_{12}(T)\le \theta.
\end{equation}
\end{proposition}

\begin{proof}
The constraint $B_{12}(T)\le\theta$ is equivalent to
\begin{equation*}
\sum_{v\in T}w_v^{12}
\ge
-\log\theta.
\end{equation*}
Let $\tau=-\log\theta$. Among all subsets of cardinality $q$, the ordered prefix $T_q(\Rcal)$ maximises the sum of the weights $\sum_{v\in T}w_v^{12}$, because the weights are arranged in non-increasing order. By definition, $q^\star$ is the smallest prefix length such that
\begin{equation*}
\sum_{\ell=1}^{q^\star}w_{v_{(\ell)}}^{12}
\ge
\tau.
\end{equation*}
Hence, for every $q<q^\star$,
\begin{equation*}
\sum_{\ell=1}^{q}w_{v_{(\ell)}}^{12}
<
\tau.
\end{equation*}
Since no subset of cardinality $q$ can have a larger weight sum than the first $q$ ranked variables, no subset with cardinality smaller than $q^\star$ can satisfy $B_{12}(T)\le\theta$. The prefix $T^\star=T_{q^\star}(\Rcal)$ satisfies the threshold condition by construction, and therefore has minimum cardinality.
\end{proof}

\subsection{Algorithm}

Algorithm~\ref{alg:calibrated_overlap} implements the calibrated residual-overlap rule. For each class pair, it accumulates the log-separation contributed by the variables as they appear in the ranking. The procedure stops at the first prefix for which every accumulated value reaches the required level $-\log\theta$. We assume throughout that $k \ge 2$ and that all class priors are positive, and use the non-trivial case $\theta\in(0,1)$, so that the empty prefix cannot satisfy the stopping rule. 

If no prefix satisfies all pairwise conditions, the feasibility condition in Proposition~\ref{prop:feasibility} fails. The procedure may then return the full ranked set $T_m(\Rcal)$ together with a ``not calibrated'' status. This status indicates that at least one class pair remains above the prescribed tolerance, that is, $B_{ij}(T_m(\Rcal))>\theta$ for some $i<j$. In implementation, if $\beta_v^{ij}=0$, the corresponding accumulated value $A_{ij}$ is set to $+\infty$, reflecting complete product-model separation for that contrast.

The version with pair-specific thresholds is obtained by replacing the common requirement $A_{ij}\ge -\log\theta$ with $A_{ij}\ge -\log\theta_{ij}$ for each pair $i<j$. 

Although the selected representation is the single class-independent subset $T^\star$, the pairwise residual overlaps computed by the stopping rule provide useful diagnostic information. For each contrast $(c_i,c_j)$, one may record the first prefix length at which that contrast alone reaches its prescribed threshold. The global stopping point is then determined by the slowest calibrating contrast. These quantities identify the pairwise bottlenecks that delay stopping, without changing the fact that the feature subset used for learning remains common to all classes. They are optional diagnostic summaries rather than separate class-specific learning representations.

\begin{algorithm}[t]
\caption{Calibrated Residual-Overlap Feature Selection}
\label{alg:calibrated_overlap}
\KwIn{Ranking $\Rcal=(v_{(1)},\ldots,v_{(m)})$; pairwise overlaps
$\{\beta_v^{ij}\}$; non-trivial uniform threshold $\theta\in(0,1)$}
\KwOut{Selected subset $T^\star$}
$T\leftarrow\varnothing$\;
\For{every pair $1\le i<j\le k$}{
    $A_{ij}\leftarrow0$
    }
\For{$\ell=1$ \KwTo $m$}{
    $v\leftarrow v_{(\ell)}$\;
    $T\leftarrow T\cup\{v\}$\;
    \For{each pair $1\le i<j\le k$}{
        $A_{ij}\leftarrow A_{ij}-\log \beta_v^{ij}$\;
    }
    \If{$A_{ij}\ge -\log\theta$ for all $1\le i<j\le k$}{
        \Return{$T$}\;
    }
}
\Return{$T_m(\Rcal)$ with status ``not calibrated''}
\end{algorithm}

\subsection{Error calibration} \label{subsec:error_calibration}

The residual-overlap threshold $\theta$ can be chosen so as to control a model-based all-pairs risk bound. To make this calibration precise, define the labelled product marginal problem $\mathbb Q_T$ as the joint law of $(X_T,Y)$ with class priors $\pi_1,\ldots,\pi_k$ and class-conditional distributions $P_{1,T}^{\otimes},\ldots,P_{k,T}^{\otimes}$. That is, for every measurable set $A$ in the sample space of $X_T$,
\begin{equation}
\mathbb Q_T(Y=c_i)=\pi_i,
\qquad
\mathbb Q_T(X_T\in A\mid Y=c_i)=P_{i,T}^{\otimes}(A),
\qquad
i=1,\ldots,k.
\end{equation}
Equivalently, under $\mathbb Q_T$, the class prior is $\pi_i$ and the class-conditional law of $X_T$ given $Y\!=\!c_i$ is the product marginal $P_{i,T}^{\otimes}$.

For each pair $(c_i,c_j)$, consider the binary classification problem that distinguishes only these two classes under the product marginal model. Its class-conditional densities are $f_{i,T}^{\otimes}$ and $f_{j,T}^{\otimes}$, and the original class priors $\pi_i$ and $\pi_j$ are used. Let $h_{ij}^{\star,\otimes}$ denote the corresponding binary Bayes rule. Thus $h_{ij}^{\star,\otimes}(x)$ assigns $x$ to either $c_i$ or $c_j$, choosing the class with the larger value of $\pi_i f_{i,T}^{\otimes}(x)$ or $\pi_j f_{j,T}^{\otimes}(x)$.

The multiclass all-pairs classifier $h_{\mathrm{AP}}^{\star,\otimes}$ is then obtained by applying all binary rules $h_{ij}^{\star,\otimes}$, $1\leq i<j\leq k$, and aggregating their decisions through a fixed voting or aggregation rule. Its product-model all-pairs risk is
\begin{equation}
R_{\mathrm{AP}}^{\star,\otimes}(T)
=
\mathbb Q_T
\left\{
h_{\mathrm{AP}}^{\star,\otimes}(X_T)\ne Y
\right\}.
\label{eq:ap_product_risk_def}
\end{equation}
This quantity is the error probability, under the labelled product marginal problem $\mathbb Q_T$, of the ideal all-pairs classifier whose binary components are Bayes rules for the corresponding pairwise product-model problems. It should be distinguished from the Bayes risk of the full multiclass product problem, which would be associated with the direct multiclass Bayes rule.

The following condition ensures that correct pairwise decisions involving the true class are preserved by the multiclass aggregation rule.

\begin{assumption}[All-pairs consistency]\label{ass:all_pairs_consistency}
For every observation $x$ with true class $c_i$, suppose that all binary product-model Bayes rules involving $c_i$ assign $x$ to $c_i$; that is,
\begin{equation*}
h_{ij}^{\star,\otimes}(x)=c_i
\quad \text{for every } j\ne i .
\end{equation*}
Then the all-pairs aggregation rule must also assign $x$ to $c_i$:
\begin{equation*}
h_{\mathrm{AP}}^{\star,\otimes}(x)=c_i .
\end{equation*}
\end{assumption}

This condition is satisfied by standard one-versus-one majority voting. If the true class $c_i$ wins all its $k-1$ pairwise comparisons, then it receives $k-1$ votes. No other class can receive more than $k-2$ votes, since each other class loses its direct comparison against $c_i$. Hence majority voting assigns the final label to $c_i$.

\begin{theorem}[Class-independent all-pairs calibration] \label{thm:multiclass_calibration}
Suppose Assumption~\ref{ass:all_pairs_consistency} holds. If $T\subseteq\Fcal$ satisfies
\begin{equation}
B_{ij}(T)
\le
\theta,
\qquad
1\le i<j\le k,
\end{equation}
for a uniform residual-overlap threshold $\theta\in(0,1]$, then
\begin{equation}
R_{\mathrm{AP}}^{\star,\otimes}(T)
\le
\theta
\sum_{1\le i<j\le k}\sqrt{\pi_i\pi_j}.
\label{eq:multiclass_error_bound}
\end{equation}
Consequently,
\begin{equation}
R_{\mathrm{AP}}^{\star,\otimes}(T)
\le
\frac{k-1}{2}\theta.
\label{eq:multiclass_prior_free_bound}
\end{equation}
\end{theorem}

\begin{proof}
For each pair $i<j$, define the pairwise error event
\begin{equation*}
\Ecal_{ij}
=
\left\{
Y=c_i,\ h_{ij}^{\star,\otimes}(X_T)\ne c_i
\right\}
\cup
\left\{
Y=c_j,\ h_{ij}^{\star,\otimes}(X_T)\ne c_j
\right\}.
\end{equation*}
By Assumption~\ref{ass:all_pairs_consistency}, an all-pairs classification error can occur only if at least one pairwise comparison involving the true class is incorrect. Hence
\begin{equation*}
\left\{
h_{\mathrm{AP}}^{\star,\otimes}(X_T)\ne Y
\right\}
\subseteq
\bigcup_{1\le i<j\le k}\Ecal_{ij}.
\end{equation*}
The union bound gives
\begin{equation*}
R_{\mathrm{AP}}^{\star,\otimes}(T)
\le
\sum_{1\le i<j\le k}\mathbb Q_T(\Ecal_{ij}).
\end{equation*}
For the binary Bayes rule associated with classes $c_i$ and $c_j$,
\begin{equation*}
\mathbb Q_T(\Ecal_{ij})
=
\int
\min
\left\{
\pi_i f_{i,T}^{\otimes}(x),
\pi_j f_{j,T}^{\otimes}(x)
\right\}
\,d\mu_T(x).
\end{equation*}
By the pairwise Bhattacharyya bound in Lemma~\ref{lem:pairwise_bound},
\begin{equation*}
\mathbb Q_T(\Ecal_{ij})
\le
\sqrt{\pi_i\pi_j}\,B_{ij}(T).
\end{equation*}
Therefore,
\begin{equation*}
R_{\mathrm{AP}}^{\star,\otimes}(T)
\le
\sum_{1\le i<j\le k}
\sqrt{\pi_i\pi_j}\,B_{ij}(T)
\le
\theta
\sum_{1\le i<j\le k}\sqrt{\pi_i\pi_j}.
\end{equation*}
Finally,
\begin{equation*}
\left(
\sum_{i=1}^k\sqrt{\pi_i}
\right)^2
=
\sum_{i=1}^k\pi_i
+
2\sum_{1\le i<j\le k}\sqrt{\pi_i\pi_j}
=
1+
2\sum_{1\le i<j\le k}\sqrt{\pi_i\pi_j}.
\end{equation*}
By Cauchy's inequality,
\begin{equation*}
\left(
\sum_{i=1}^k\sqrt{\pi_i}
\right)^2
\le
k\sum_{i=1}^k\pi_i
=
k.
\end{equation*}
Hence
\begin{equation*}
\sum_{1\le i<j\le k}\sqrt{\pi_i\pi_j}
\le
\frac{k-1}{2},
\end{equation*}
which proves \eqref{eq:multiclass_prior_free_bound}.
\end{proof}

\begin{corollary}[Prior-dependent threshold calibration]\label{cor:prior_dependent_calibration}
Let $\varepsilon>0$ be a target product-model all-pairs risk level. If the uniform residual-overlap threshold $\theta\in(0,1]$ satisfies
\begin{equation}
\theta
\le
\frac{\varepsilon}
{\sum_{1\le i<j\le k}\sqrt{\pi_i\pi_j}},
\label{eq:cor_prior_dependent}
\end{equation}
then every subset satisfying
\begin{equation}
B_{ij}(T)\le\theta,
\qquad
1\le i<j\le k,
\end{equation}
also satisfies
\begin{equation}
R_{\mathrm{AP}}^{\star,\otimes}(T)\le \varepsilon.
\end{equation}
\end{corollary}

\begin{proof}
The result follows directly from Theorem~\ref{thm:multiclass_calibration}.
\end{proof}

The same argument also permits pair-specific thresholds. If
\begin{equation*}
B_{ij}(T)\le \theta_{ij},
\qquad 1\le i<j\le k,
\end{equation*}
then
\begin{equation*}
R_{\operatorname{AP}}^{\star,\otimes}(T)
\le
\sum_{1\le i<j\le k}\sqrt{\pi_i\pi_j}\,\theta_{ij}.
\end{equation*}
Thus pair-specific tolerances are admissible whenever
\begin{equation*}
\sum_{1\le i<j\le k}\sqrt{\pi_i\pi_j}\,\theta_{ij}
\le
\varepsilon.
\end{equation*}

\begin{corollary}[Prior-free threshold calibration] \label{cor:prior_free_calibration}
Let $\varepsilon>0$ be a target product-model all-pairs risk level. If the uniform residual-overlap threshold $\theta\in(0,1]$ satisfies
\begin{equation}
\theta
\le
\frac{2\varepsilon}{k-1},
\label{eq:cor_prior_free}
\end{equation}
then every subset satisfying
\begin{equation}
B_{ij}(T)\le\theta,
\qquad
1\le i<j\le k,
\end{equation}
also satisfies
\begin{equation}
R_{\mathrm{AP}}^{\star,\otimes}(T)\le \varepsilon
\end{equation}
for every class-prior vector $(\pi_1,\ldots,\pi_k)$.
\end{corollary}

\begin{proof}
The result follows from the prior-free bound \eqref{eq:multiclass_prior_free_bound}.
\end{proof}

\subsection{Choosing the residual-overlap threshold} \label{subsec:choosing_threshold}

The stopping rule is controlled by the residual-overlap threshold $\theta$. This parameter acts on the overlap scale: it is the numerical level imposed on each pairwise product overlap $B_{ij}(T)$. It should be distinguished from the risk-level parameter $\varepsilon$, which denotes a target upper bound on the product-model all-pairs risk $R_{\mathrm{AP}}^{\star,\otimes}(T)$. Thus $\theta$ is the operational threshold used by the stopping rule, whereas $\varepsilon$ is specified on the product-model risk scale.

Let
\begin{equation*}
S_\pi
=
\sum_{1\le i<j\le k}\sqrt{\pi_i\pi_j}.
\label{eq:S_pi}
\end{equation*}
By Theorem~\ref{thm:multiclass_calibration}, if a subset $T$ satisfies
\begin{equation*}
B_{ij}(T)\le \theta,
\qquad
1\le i<j\le k,
\end{equation*}
then
\begin{equation*}
R_{\mathrm{AP}}^{\star,\otimes}(T)
\le
\theta S_\pi.
\label{eq:risk_threshold_relation}
\end{equation*}
This identity gives two equivalent ways to use the calibration.

First, one may fix the residual-overlap threshold $\theta$ and ask which product-model risk level is guaranteed. In that case,
\begin{equation}
\varepsilon(\theta)
=
\theta S_\pi.
\label{eq:theta_to_risk}
\end{equation}
Therefore every subset satisfying $B_{ij}(T)\le\theta$ for all pairs also satisfies
\begin{equation}
R_{\mathrm{AP}}^{\star,\otimes}(T)
\le
\varepsilon(\theta).
\end{equation}

Conversely, one may specify the desired target all-pairs risk level $\varepsilon>0$ first and derive a sufficient residual-overlap threshold. If the class priors $\pi_1,\ldots,\pi_k$ are known or reliably estimated, the prior-dependent choice is
\begin{equation}
\theta_\varepsilon
=
\min
\left\{
1,
\frac{\varepsilon}{S_\pi}
\right\}.
\label{eq:risk_to_theta}
\end{equation}
With this choice, any subset satisfying
\begin{equation}
B_{ij}(T)\le\theta_\varepsilon,
\qquad
1\le i<j\le k,
\end{equation}
also satisfies the target product-model all-pairs risk bound
\begin{equation}
R_{\mathrm{AP}}^{\star,\otimes}(T)
\le
\varepsilon.
\end{equation}
Thus \eqref{eq:theta_to_risk} answers what risk level is guaranteed by a chosen threshold, whereas \eqref{eq:risk_to_theta} answers which threshold is sufficient to guarantee a chosen risk level.

If a prior-free calibration is desired, Corollary~\ref{cor:prior_free_calibration} gives the sufficient condition
\begin{equation}
\theta
\le
\frac{2\varepsilon}{k-1}.
\end{equation}
Equivalently, a prior-free operational choice is
\begin{equation}
\theta_\varepsilon^{\mathrm{pf}}
=
\min
\left\{
1,
\frac{2\varepsilon}{k-1}
\right\}.
\label{eq:theta_prior_free_trunc}
\end{equation}
This choice guarantees the target all-pairs risk level for every possible class-prior vector. It is generally more conservative than the prior-dependent calibration when reliable prior information is available.

In the pair-specific case, the target level $\varepsilon$ may be allocated unevenly across class contrasts by choosing thresholds $\theta_{ij}$ satisfying
\begin{equation*}
\sum_{1\le i<j\le k}\sqrt{\pi_i\pi_j}\,\theta_{ij}
\le
\varepsilon.
\end{equation*}
The uniform rule corresponds to the special case $\theta_{ij}=\theta$ for all pairs. Pair-specific thresholds may be useful when some contrasts have greater scientific, clinical or operational importance than others.

\subsection{Estimating marginal overlaps}

In practice, the marginal Bhattacharyya coefficients $\beta_v^{ij}$ are unknown and must be estimated from the training data. Let $\widehat P_{i,v}$ be an estimator of the class-conditional marginal law $P_{i,v}$, with density $\widehat f_{i,v}$ with respect to the same dominating measure $\mu_v$. The plug-in overlap coefficient is defined by
\begin{equation}\label{eq:plugin_beta}
\widehat\beta_v^{ij}
=
BC(\widehat P_{i,v},\widehat P_{j,v})
=
\int
\sqrt{\widehat f_{i,v}(x)\widehat f_{j,v}(x)}
\,d\mu_v(x),
\qquad
\widehat w_v^{ij}
=
-\log \widehat\beta_v^{ij}.
\end{equation}
For a prefix $T_q(\Rcal)=\{v(1),\ldots,v(q)\}$, the corresponding plug-in residual overlap is
\begin{equation}
\widehat B_{ij}(T_q(\Rcal))
=
\prod_{\ell=1}^{q}
\widehat\beta_{v(\ell)}^{ij}.
\end{equation}

The estimators $\widehat P_{i,v}$ may be empirical mass functions for discrete variables, parametric marginal models, kernel density estimators, or other univariate density estimators for continuous variables. In particular, kernel density estimation provides a flexible non-parametric way to estimate class-conditional marginal densities before computing the plug-in Bhattacharyya coefficient in \eqref{eq:plugin_beta}. Related KDE-based class-conditional density estimates have recently been used in explainable Naive Bayes classification for high-dimensional genomic data \cite{AguilarRuiz2024XNB}. The stopping rule used in data analysis is obtained by replacing the population overlaps by these plug-in estimates. Its empirical reliability therefore depends on the quality of the marginal overlap estimates, especially in small samples, rare classes or continuous settings. Conservative variants may use regularised density estimates, upper confidence bounds for $\beta_v^{ij}$, sample splitting or resampling diagnostics to reduce overly optimistic stopping decisions.

\subsection{Dependence and the model-based calibration gap} \label{subsec:dependence_calibration_gap}

The calibration developed above is model-based. It is valid for the labelled product marginal problem, but it is not, by itself, a solution to the general problem of dependence among variables. Let
\begin{equation*}
P_{i,T}=\mathcal L(X_T\mid Y=c_i)
\end{equation*}
denote the true class-conditional joint distribution on the selected variables, and let
\begin{equation*}
Q_{i,T}=P_{i,T}^{\otimes}=\bigotimes_{v\in T}P_{i,v}
\end{equation*}
be the corresponding product marginal model. The stopping rule calibrates the product-model residual overlap
\begin{equation*}
B_{ij}^{\otimes}(T)
=
\operatorname{BC}(Q_{i,T},Q_{j,T})
=
\prod_{v\in T}\beta_v^{ij}.
\end{equation*}
In general,
\begin{equation*}
\operatorname{BC}(P_{i,T},P_{j,T})
\ne
B_{ij}^{\otimes}(T).
\end{equation*}

Equality is guaranteed if the true joint class-conditional laws coincide with their product marginal versions for the two classes under consideration, for example under conditional independence of the selected variables within each class. Otherwise, the product construction should be interpreted as an explicit working model for marginal evidence accumulation, not as a claim that the true joint distributions factorise.

This distinction is important. If selected variables are redundant, the product model may count largely duplicated marginal evidence more than once, thereby underestimating the true joint overlap. Conversely, if discrimination is carried primarily by dependence structure rather than by marginal distributions, the product model may fail to detect useful joint information. Thus the proposed residual-overlap threshold calibrates accumulated marginal evidence under an explicit product construction.

The risk bounds derived above should therefore be read as product-model calibrations. They control an ideal all-pairs risk under the labelled product marginal problem, not the risk of an arbitrary classifier under the true joint distribution. For the latter, an additional model gap must be considered, together with ordinary estimation and learning error. Accordingly, the residual-overlap threshold should be interpreted as a model-based calibration of accumulated marginal evidence, not as a distribution-free guarantee for arbitrary dependent feature vectors.

\subsection{Advantages, scope and limitations} \label{subsec:advantages_scope_limitations}

The proposed method provides a calibrated stopping principle for supervised feature rankings. Its main contribution is to replace a fixed top-$q$ choice by a residual-overlap threshold with an explicit interpretation under the labelled product marginal model. The selected prefix is the first point along the ranking at which all pairwise class contrasts have accumulated enough marginal separation to satisfy the prescribed overlap condition. Hence the retained cardinality is determined by a distributional criterion rather than by an externally imposed subset size.

The procedure is modular with respect to the input ranking. It may be applied to rankings obtained from domain knowledge, mutual information, a $\chi^2$ score, a statistical test, model-based importance measures, distributional criteria or other supervised relevance scores. The purpose of the method is not to define a new ranking score, but to attach a calibrated truncation rule to an existing ordered list and thereby convert it into a class-independent feature subset. Although the stopping condition is based on pairwise class-conditional overlaps, the output is a single subset $T^\star\subseteq\mathcal F$ shared by all classes. Pairwise quantities may still be used diagnostically to identify the contrasts that delay stopping, but they do not produce separate class-dependent feature spaces.

The rule is interpretable and computationally simple. The weights $w_v^{ij}=-\log\beta_v^{ij}$ give an additive account of the marginal separation contributed by variable $v$ to contrast $(c_i,c_j)$. Thus the stopping condition can be read either multiplicatively, in terms of residual overlaps, or additively, in terms of accumulated log-separation. Once the one-dimensional class-conditional overlaps have been estimated, evaluating the stopping rule along a ranking has cost proportional to the number of ranked variables and to the number of class pairs.

These advantages hold within the scope of the product marginal model. The Bhattacharyya factorisation is exact for the product laws $P_{i,T}^{\otimes}$, but not necessarily for the true class-conditional joint laws $P_{i,T}$. Consequently, the calibration controls a model-based all-pairs risk quantity, rather than the risk of an arbitrary classifier under the true joint distribution. Strong dependence among selected variables may create a calibration gap. Redundant variables can make the product overlap decrease too quickly, whereas discriminatory dependence structures with weak marginal signals may be missed.

A related practical issue is redundancy in the input ranking. Highly ranked variables may reflect the same discriminatory signal, so marginal relevance does not necessarily imply non-redundant joint usefulness. Redundancy-aware filters, including correlation-based feature selection, fast correlation-based filtering, minimum-redundancy maximum-relevance selection, Markov-blanket approximations and conditional-information criteria, were developed to address this distinction \cite{Hall1999,KollerSahami1996,YuLiu2003,YuLiu2004,Peng2005,Brown2012}. When redundancy is a concern, the ranking may therefore be refined before applying the residual-overlap rule, for example by removing, merging or deprioritising variables that are strongly redundant with earlier ones. This can improve interpretability, reduce measurement or storage costs, increase stability and lessen repeated counting of the same marginal evidence, particularly in high-dimensional biological data with co-expressed genes or probes \cite{Guyon2003,Saeys2007}.

The empirical reliability of the rule also depends on the estimation of the marginal overlaps. In small samples, especially with continuous variables or rare classes, density and overlap estimates may be unstable. Regularised density estimation, conservative upper confidence bounds for $\beta_v^{ij}$, sample splitting or resampling diagnostics can help reduce overly optimistic stopping decisions.

These limitations define the scope of the method. The proposed rule is best viewed as a reproducible, product-model-calibrated stopping principle for supervised feature rankings.

\section{Comparative analysis}\label{sec:comparative_analysis}

We evaluated the proposed residual-overlap stopping rule on 18 high-dimensional genomic datasets from CuMiDa \cite{Feltes2019}. The datasets contain between 22,277 and 54,675 variables.

The target product-model all-pairs risk level was fixed at $\varepsilon\!=\!0.001$, and the uniform residual-overlap calibration was used throughout the comparison. Thus, for each dataset the operational residual-overlap threshold was computed according to Eq.~(\ref{eq:risk_to_theta}). Since this threshold depends on the empirical class-prior distribution through the quantity $\sum_{i<j}\sqrt{\pi_i\pi_j}$, the resulting value of $\theta$ varies across datasets. Binary and approximately balanced datasets yield thresholds close to $0.002$. Multiclass datasets, especially those with several classes carrying substantial prior mass, have larger values of $S_\pi$ and therefore require smaller residual-overlap thresholds for the same target level $\varepsilon$.

All reported results were obtained using stratified 10-fold cross-validation. The marginal Bhattacharyya overlaps, the feature ranking and the stopping point were also estimated exclusively from the training data. For each training fold, the class-conditional marginal distribution of each variable was estimated separately within each class using Gaussian kernel density estimation. For every variable--class pair, the density was evaluated on a fixed grid of 50 support points spanning the observed training range of that variable, and the bandwidth was selected using Silverman's rule of thumb. The resulting univariate density estimates were then used to compute the plug-in Bhattacharyya coefficients for each pairwise class contrast. All these operations were performed within the training fold only, before determining the
feature ranking and the calibrated stopping point.

Two external ranking criteria are reported: mutual information (MI) and the chi-square score ($\chi^2$). The column ``All'' denotes the all-features baseline, where the classifier is trained using all available variables. Although All Features is not a ranking method, it provides the natural reference for assessing whether the proposed stopping rule preserves predictive performance after dimensionality reduction. Two downstream classifiers were considered: Gaussian naive Bayes (GNB), which is naturally aligned with the product-marginal interpretation of the framework, and logistic regression (LR), which provides a discriminative linear baseline. Although this comparison reports accuracy and macro-F1 for compactness, the selected subsets could also be evaluated using richer multiclass performance summaries, particularly probability-based assessments that can highlight class-dependent performance patterns beyond scalar metrics \cite{AguilarRuizMichalak2022MCP}.

Table~\ref{tab:num_features_eps0001_uniform} summarizes the number of selected variables. The dimensionality reduction is substantial. The all-features baseline uses, on average, 42,109 variables. In contrast, the residual-overlap stopping rule selects only 42.9 variables on average when applied to the MI ranking and 56.0 variables when applied to the $\chi^2$ ranking. These values correspond to approximately $0.10\%$ and $0.13\%$ of the original feature space, respectively. Therefore, the proposed stopping rule reduces the dimensionality by roughly three orders of magnitude while retaining an explicit calibration criterion linked to the product-model all-pairs risk level.

The selected dimensionality is also adaptive across datasets. Some datasets, such as Breast\_GSE42568, Lung\_GSE7670, Pancreatic\_GSE16515, Prostate\_GSE46602 and Throat\_GSE59102, require only a few variables to reach the prescribed residual-overlap threshold. In contrast, more heterogeneous or multiclass datasets, such as Breast\_GSE26304 and Leukemia\_GSE28497, require longer prefixes. This behaviour is consistent with the theoretical role of the stopping rule: the selected prefix length is not fixed a priori, but is determined by the amount of accumulated marginal evidence needed to satisfy the pairwise overlap constraints.

Table~\ref{tab:performance_eps0001_uniform} reports the predictive performance obtained with the selected subsets and with the all-features baseline. For GNB, the all-features baseline achieves a mean accuracy of $0.809$ and a mean macro-F1 of $0.752$. The MI-based residual-overlap selection obtains a higher mean accuracy, $0.825$, and a higher mean macro-F1, $0.795$, while using only 42.9 variables on average. The $\chi^2$-based selection obtains a mean accuracy of $0.805$ and a mean macro-F1 of $0.779$, also with a very small selected subset.

\begin{table}[t]
\centering
\scriptsize
\setlength{\tabcolsep}{8pt}
\fontsize{8}{10}\selectfont
\caption{Number of available features, selected prefix lengths and prior-dependent uniform residual-overlap thresholds for $\varepsilon=0.001$.}
\label{tab:num_features_eps0001_uniform}
\begin{tabular}{lrrrr}
\toprule
\multirow{2}{*}{Dataset} & \multicolumn{3}{c}{Num. Features ($q^\star$)} & \\
\cmidrule(lr){2-4}
& \multicolumn{1}{c}{All} & \multicolumn{1}{c}{MI} & \multicolumn{1}{c}{$\chi^2$} & \multicolumn{1}{c}{$\theta$} \\
\midrule
Bladder\_GSE31189 & 54,675 & 122.0 & 73.9 & 0.002017 \\
Brain\_GSE15824 & 54,675 & 18.6 & 34.0 & 0.000676 \\
Brain\_GSE50161 & 54,675 & 31.1 & 42.3 & 0.000537 \\
Breast\_GSE10797 & 22,277 & 30.4 & 76.0 & 0.001074 \\
Breast\_GSE26304 & 33,637 & 148.7 & 329.3 & 0.000754 \\
Breast\_GSE42568 & 54,675 & 4.1 & 2.8 & 0.002980 \\
Breast\_GSE45827 & 54,675 & 39.0 & 54.7 & 0.000433 \\
Breast\_GSE7904 & 54,675 & 7.0 & 6.7 & 0.001073 \\
Colorectal\_GSE21510 & 54,675 & 8.1 & 19.2 & 0.001273 \\
Colorectal\_GSE77953 & 22,283 & 46.4 & 90.9 & 0.000676 \\
Leukemia\_GSE28497 & 22,283 & 199.7 & 167.4 & 0.000361 \\
Lung\_GSE7670 & 22,283 & 5.4 & 4.3 & 0.002003 \\
Ovary\_GSE6008 & 22,283 & 49.1 & 41.9 & 0.000770 \\
Pancreatic\_GSE16515 & 54,675 & 8.4 & 8.1 & 0.002195 \\
Prostate\_GSE11682 & 33,467 & 29.5 & 28.5 & 0.002001 \\
Prostate\_GSE46602 & 54,675 & 6.2 & 5.1 & 0.002214 \\
Renal\_GSE53757 & 54,675 & 13.1 & 14.5 & 0.002000 \\
Throat\_GSE59102 & 32,703 & 6.2 & 7.9 & 0.002163 \\
\midrule
\textbf{Mean} & \textbf{42,109} & \textbf{42.9} & \textbf{56.0} & \textbf{0.001400} \\
\bottomrule
\end{tabular}
\vspace{1mm}
\end{table}

\begin{table*}[t]
\centering
\scriptsize
\setlength{\tabcolsep}{6pt}
\fontsize{8}{10}\selectfont
\caption{Predictive performance with $\varepsilon=0.001$ and uniform residual-overlap calibration.}
\label{tab:performance_eps0001_uniform}
\begin{tabular}{lrrrrrrrrrrrr}
\toprule
\multirow{3}{*}{Dataset}
& \multicolumn{6}{c}{GNB}
& \multicolumn{6}{c}{LR} \\
\cmidrule(lr){2-7}\cmidrule(lr){8-13}
& \multicolumn{3}{c}{Accuracy}
& \multicolumn{3}{c}{Macro-F1}
& \multicolumn{3}{c}{Accuracy}
& \multicolumn{3}{c}{Macro-F1} \\
\cmidrule(lr){2-4}\cmidrule(lr){5-7}
\cmidrule(lr){8-10}\cmidrule(lr){11-13}
& All & MI & $\chi^2$
& All & MI & $\chi^2$
& All & MI & $\chi^2$
& All & MI & $\chi^2$ \\
\midrule
Bladder\_GSE31189 & 0.567 & 0.567 & 0.532 & 0.530 & 0.551 & 0.515 & 0.692 & 0.551 & 0.681 & 0.674 & 0.524 & 0.659 \\
Brain\_GSE15824 & 0.724 & 0.700 & 0.724 & 0.608 & 0.580 & 0.635 & 0.805 & 0.757 & 0.671 & 0.767 & 0.682 & 0.612 \\
Brain\_GSE50161 & 0.908 & 0.908 & 0.831 & 0.864 & 0.863 & 0.799 & 0.938 & 0.931 & 0.854 & 0.937 & 0.932 & 0.815 \\
Breast\_GSE10797 & 0.712 & 0.719 & 0.726 & 0.567 & 0.634 & 0.716 & 0.802 & 0.748 & 0.807 & 0.786 & 0.715 & 0.791 \\
Breast\_GSE26304 & 0.392 & 0.443 & 0.461 & 0.329 & 0.333 & 0.478 & 0.305 & 0.410 & 0.418 & 0.295 & 0.468 & 0.406 \\
Breast\_GSE42568 & 0.992 & 0.983 & 0.973 & 0.948 & 0.929 & 0.909 & 0.965 & 0.965 & 0.957 & 0.930 & 0.907 & 0.884 \\
Breast\_GSE45827 & 0.927 & 0.888 & 0.868 & 0.919 & 0.891 & 0.881 & 0.934 & 0.934 & 0.914 & 0.935 & 0.935 & 0.928 \\
Breast\_GSE7904 & 0.935 & 0.912 & 0.956 & 0.872 & 0.888 & 0.952 & 0.935 & 0.956 & 0.956 & 0.924 & 0.943 & 0.952 \\
Colorectal\_GSE21510 & 0.993 & 0.966 & 0.891 & 0.987 & 0.931 & 0.796 & 0.946 & 0.959 & 0.911 & 0.929 & 0.908 & 0.822 \\
Colorectal\_GSE77953 & 0.830 & 0.850 & 0.743 & 0.806 & 0.848 & 0.695 & 0.980 & 0.980 & 0.780 & 0.967 & 0.970 & 0.723 \\
Leukemia\_GSE28497 & 0.833 & 0.822 & 0.843 & 0.781 & 0.800 & 0.829 & 0.876 & 0.850 & 0.875 & 0.864 & 0.822 & 0.856 \\
Lung\_GSE7670 & 0.900 & 0.960 & 0.960 & 0.898 & 0.960 & 0.960 & 0.960 & 0.940 & 0.980 & 0.956 & 0.936 & 0.980 \\
Ovary\_GSE6008 & 0.753 & 0.713 & 0.764 & 0.685 & 0.731 & 0.800 & 0.712 & 0.692 & 0.784 & 0.696 & 0.737 & 0.793 \\
Pancreatic\_GSE16515 & 0.847 & 0.927 & 0.903 & 0.760 & 0.919 & 0.860 & 0.827 & 0.847 & 0.903 & 0.784 & 0.834 & 0.894 \\
Prostate\_GSE11682 & 0.542 & 0.667 & 0.633 & 0.425 & 0.623 & 0.572 & 0.700 & 0.767 & 0.800 & 0.648 & 0.732 & 0.738 \\
Prostate\_GSE46602 & 0.915 & 0.980 & 0.960 & 0.840 & 0.976 & 0.956 & 0.920 & 0.960 & 0.920 & 0.877 & 0.956 & 0.912 \\
Renal\_GSE53757 & 0.839 & 0.853 & 0.839 & 0.838 & 0.852 & 0.837 & 0.852 & 0.866 & 0.852 & 0.851 & 0.865 & 0.851 \\
Throat\_GSE59102 & 0.950 & 1.000 & 0.880 & 0.886 & 1.000 & 0.843 & 1.000 & 0.950 & 0.885 & 1.000 & 0.947 & 0.846 \\
\midrule
\textbf{Mean} & \textbf{0.809} & \textbf{0.825} & \textbf{0.805} & \textbf{0.752} & \textbf{0.795} & \textbf{0.780} & \textbf{0.842} & \textbf{0.837} & \textbf{0.830} & \textbf{0.823} & \textbf{0.823} & \textbf{0.803} \\
\bottomrule
\end{tabular}
\vspace{1mm}
\end{table*}

For LR, the all-features baseline gives the highest mean accuracy, $0.842$, but the MI-based selected subsets obtain a very similar value, $0.837$, and the same mean macro-F1, $0.823$. The $\chi^2$ ranking yields mean accuracy $0.830$ and mean macro-F1 $0.803$. Thus, at the descriptive level, the reduced subsets retain almost all the predictive performance of the full feature space, and in some cases even obtain higher average scores. However, these mean differences should not be interpreted as evidence of systematic superiority without formal statistical testing.

To assess whether the observed differences were statistically significant, we used non-parametric paired tests across datasets. For each classifier and each metric, we first applied the Friedman test to the three alternatives: All, MI and $\chi^2$. The null hypothesis is that the three alternatives have equal performance distributions across datasets. The results are shown in Table~\ref{tab:friedman_tests_eps0001_uniform}.

\begin{table}[t]
\centering
\scriptsize
\setlength{\tabcolsep}{8pt}
\fontsize{8}{10}\selectfont
\caption{Friedman tests comparing All, MI and $\chi^2$ across datasets for $\varepsilon=0.001$ and uniform calibration.}
\label{tab:friedman_tests_eps0001_uniform}
\begin{tabular}{llrr}
\toprule
Classifier & Metric & Friedman statistic & $p$-value \\
\midrule
GNB & Accuracy & 1.088 & 0.580 \\
GNB & Macro-F1 & 3.916 & 0.141 \\
LR & Accuracy & 0.377 & 0.828 \\
LR & Macro-F1 & 0.444 & 0.801 \\
\bottomrule
\end{tabular}
\vspace{1mm}
\end{table}

None of the Friedman tests rejects the null hypothesis at the $0.05$ significance level. Therefore, for both classifiers and both metrics, the observed differences among All, MI and $\chi^2$ are not statistically significant across the 18 datasets. Pairwise Wilcoxon signed-rank tests with Holm correction led to the same conclusion. All adjusted $p$-values were above $0.05$; the smallest adjusted value was obtained for the comparison between All and MI under GNB macro-F1. Hence, the apparent descriptive advantages observed in Table~\ref{tab:performance_eps0001_uniform} should not be interpreted as statistically significant improvements.

This result is central to the empirical interpretation of the method. The aim of the residual-overlap stopping rule is not to guarantee that the selected subset will outperform the full feature space in predictive accuracy. Rather, its purpose is to determine a calibrated truncation point that preserves predictive performance while producing a substantially smaller set of variables. From this perspective, the absence of significant differences with respect to All Features is a favourable result: the reduced subsets achieve statistically comparable performance while using only a small fraction of the original variables.

We also examined whether the two external ranking criteria, MI and $\chi^2$, produced statistically different predictive performance under the proposed stopping rule. Pairwise Wilcoxon signed-rank tests were applied across the 18 datasets, separately for each classifier and metric. No comparison between MI and $\chi^2$ was statistically significant. For GNB, MI obtained slightly higher mean accuracy than $\chi^2$ ($0.825$ vs. $0.805$), but the difference was not significant ($p\!=\!0.142$, Holm-adjusted $p\!=\!0.568$). For macro-F1 under GNB, the difference was also non-significant ($0.795$ vs. $0.779$, $p\!=\!0.538$). The same conclusion was obtained for LR, where MI and $\chi^2$ showed no significant differences in either accuracy ($p\!=\!0.850$) or macro-F1 ($p\!=\!0.408$).

Overall, these results support the main empirical claim of the paper. At the fixed target risk level $\varepsilon\!=\!0.001$, the residual-overlap stopping rule reduces the number of variables from tens of thousands to a few dozen, while maintaining predictive performance that is statistically indistinguishable from the all-features baseline in the reported tests. Moreover, MI and $\chi^2$ do not differ significantly from each other, suggesting that, within the two ranking criteria examined here, the stopping rule produced statistically comparable predictive performance. Thus, the principal empirical contribution is a calibrated and reproducible stopping criterion that achieves extreme dimensionality reduction without a detectable loss in accuracy or macro-F1.

\section{Conclusions} \label{sec:conclusions}

This work has introduced a calibrated stopping rule for ranking-based feature selection. The aim is to give a statistical interpretation to the point at which an existing feature ranking is truncated. Given a feature ranking, the proposed rule selects the shortest prefix whose residual product overlap is below a prescribed threshold for every relevant pairwise class contrast. Thus the retained cardinality is determined by a distributional calibration criterion, rather than by a fixed top-$q$ choice, an empirical cut-off, or cross-validation alone.

The construction is based on the exact factorisation of the Bhattacharyya coefficient under the product marginal model. This factorisation converts feature-wise class-conditional overlaps into accumulated log-separation scores, making the stopping rule both multiplicative in residual overlap and additive in marginal evidence. The residual-overlap threshold $\theta$ can be chosen from a target product-model all-pairs risk level $\varepsilon$, using either prior-dependent or prior-free calibration. More generally, pair-specific thresholds can be derived when the desired risk level is allocated differently across class contrasts.

The selected subset is class-independent. Pairwise class-conditional overlaps are used to monitor which contrasts have accumulated sufficient separation, but the output is a single feature subset shared by all classes. This makes the method compatible with standard multiclass pipelines and with applications in which a common interpretable representation is required. At the same time, the pairwise quantities computed by the rule provide useful diagnostic information: they identify the bottleneck contrasts and explain which variables are needed before simultaneous calibration is achieved.

The interpretation of the method is explicitly model-based. The calibration is exact for the labelled product marginal problem, not for arbitrary dependent joint distributions. Redundant variables may make the product overlap decrease too quickly, while discriminatory dependence structures with weak marginal signals may be missed. Likewise, the derived risk bounds concern ideal product-model all-pairs classifiers and should not be read as direct finite-sample guarantees for arbitrary downstream learners. These caveats do not undermine the method; rather, they delimit its scope as a transparent and reproducible stopping principle for accumulating marginal class-separation evidence.

The approach is particularly useful in supervised high-dimensional settings where a feature ranking is available and a compact set of original variables is desired, such as biomarker discovery, medical diagnostics, chemometric profiling and other interpretable classification tasks. 

Future work should investigate dependence-aware extensions, conservative and uncertainty-aware overlap estimation, resampling-based calibration, pair-specific threshold allocation, and empirical comparisons with cross-validated top-$q$ selection, stability selection, redundancy-aware filters and distributional feature-selection baselines.

\bibliographystyle{unsrt}
\bibliography{bibliography}

\newpage
\appendix

\section{A worked four-class example} \label{app:worked_four_class_example}

This appendix illustrates the stopping rule in a simple four-class setting. The example is artificial and is intended to illustrate three aspects of the method: accumulation of residual overlaps, selection of the first calibrated prefix, and interpretation of the residual-overlap threshold on the product-model risk scale.

Consider four classes $c_1,c_2,c_3,c_4$ and a ranked list of variables
\begin{equation*}
R=(v_1,v_2,\ldots,v_{100}).
\end{equation*}
There are six pairwise class contrasts. Suppose that the pairwise Bhattacharyya coefficients $\beta_v^{ij}$ have been estimated for every ranked variable $v\in\{v_1,\ldots,v_{100}\}$ and every class pair $i<j$. For compactness, Table~\ref{tab:worked_example_coefficients} displays only the first five ranked variables and the last ranked variable.

\begin{table}[!htbp]
\caption{Extract of the variable-wise Bhattacharyya coefficients for the four-class example. Coefficients are assumed to be available for all 100 ranked variables.}
\label{tab:worked_example_coefficients}
\begin{center}
\begin{tabular}{c|cccccc}
\toprule
 & $(c_1,c_2)$ & $(c_1,c_3)$ & $(c_1,c_4)$
 & $(c_2,c_3)$ & $(c_2,c_4)$ & $(c_3,c_4)$\\
\midrule
$v_1$ & 0.04 & 0.05 & 0.06 & 0.08 & 0.07 & 0.09\\
$v_2$ & 0.20 & 0.10 & 0.15 & 0.30 & 0.25 & 0.20\\
$v_3$ & 0.10 & 0.20 & 0.10 & 0.50 & 0.40 & 0.30\\
$v_4$ & 0.75 & 0.50 & 0.70 & 0.08 & 0.10 & 0.15\\
$v_5$ & 0.80 & 0.55 & 0.65 & 0.85 & 0.75 & 0.90\\
$\vdots$ & $\vdots$ & $\vdots$ & $\vdots$ & $\vdots$ & $\vdots$ & $\vdots$\\
$v_{100}$ & 0.99 & 0.96 & 0.95 & 0.97 & 0.92 & 0.94\\
\bottomrule
\end{tabular}
\end{center}
\end{table}

Small coefficients indicate strong marginal separation for the corresponding class contrast, whereas coefficients close to one indicate substantial overlap. For each prefix $T_q=\{v_1,\ldots,v_q\}$, the residual product overlap for contrast $(c_i,c_j)$ is obtained by multiplying the corresponding coefficients up to $q$:
\begin{equation*}
B_{ij}(T_q)=\prod_{\ell=1}^{q}\beta_{v_\ell}^{ij}.
\end{equation*}
Table~\ref{tab:worked_example_residuals} reports the resulting residual overlaps for the first five prefixes, which are sufficient to identify the stopping point in this example.

\begin{table}[!htbp]
\caption{Residual product overlaps along the first five prefixes.}
\label{tab:worked_example_residuals}
\begin{center}
\begin{tabular}{c|cccccc}
\toprule
$q$ & $B_{12}(T_q)$ & $B_{13}(T_q)$ & $B_{14}(T_q)$
& $B_{23}(T_q)$ & $B_{24}(T_q)$ & $B_{34}(T_q)$\\
\midrule
0 & $1$ & $1$ & $1$ & $1$ & $1$ & $1$\\
1 & $4.00\times10^{-2}$ & $5.00\times10^{-2}$ & $6.00\times10^{-2}$
  & $8.00\times10^{-2}$ & $7.00\times10^{-2}$ & $9.00\times10^{-2}$\\
2 & $8.00\times10^{-3}$ & $5.00\times10^{-3}$ & $9.00\times10^{-3}$
  & $2.40\times10^{-2}$ & $1.75\times10^{-2}$ & $1.80\times10^{-2}$\\
3 & $8.00\times10^{-4}$ & $1.00\times10^{-3}$ & $9.00\times10^{-4}$
  & $1.20\times10^{-2}$ & $7.00\times10^{-3}$ & $5.40\times10^{-3}$\\
4 & $6.00\times10^{-4}$ & $5.00\times10^{-4}$ & $6.30\times10^{-4}$
  & $9.60\times10^{-4}$ & $7.00\times10^{-4}$ & $8.10\times10^{-4}$\\
5 & $4.80\times10^{-4}$ & $2.75\times10^{-4}$ & $4.10\times10^{-4}$
  & $8.16\times10^{-4}$ & $5.25\times10^{-4}$ & $7.29\times10^{-4}$\\
  $\vdots$ & $\vdots$ & $\vdots$ & $\vdots$ & $\vdots$ & $\vdots$ & $\vdots$\\
\bottomrule
\end{tabular}
\end{center}
\end{table}

Let the common residual-overlap threshold be $\theta=10^{-3}$. The empty prefix has residual overlap equal to one for every contrast. After the first two variables, all six residual overlaps are still above $\theta$, so $T_2=\{v_1,v_2\}$ is not calibrated.

After adding $v_3$, the three contrasts involving $c_1$ have reached the threshold:
\begin{equation*}
B_{12}(T_3)=8.00\times10^{-4},\qquad B_{13}(T_3)=1.00\times10^{-3},\qquad B_{14}(T_3)=9.00\times10^{-4}.
\end{equation*}
However, the remaining contrasts are still above the threshold:
\begin{equation*}
B_{23}(T_3)=1.20\times10^{-2},\qquad B_{24}(T_3)=7.00\times10^{-3},\qquad B_{34}(T_3)=5.40\times10^{-3}.
\end{equation*}
Thus $T_3$ is not yet a calibrated prefix. The stopping rule cannot stop because the threshold must be satisfied simultaneously for all pairwise class contrasts.

After adding $v_4$, all six residual overlaps are below the common threshold:
\begin{equation*}
B_{12}(T_4)=6.00\times10^{-4},\qquad B_{13}(T_4)=5.00\times10^{-4},\qquad B_{14}(T_4)=6.30\times10^{-4}.
\end{equation*}
\begin{equation*}
B_{23}(T_4)=9.60\times10^{-4},\qquad B_{24}(T_4)=7.00\times10^{-4},\qquad B_{34}(T_4)=8.10\times10^{-4}.
\end{equation*}
Therefore, the first calibrated prefix is
\begin{equation*}
T^\star=T_4=\{v_1,v_2,v_3,v_4\}.
\end{equation*}
Although $v_5$ would further reduce the residual overlaps, it is not selected because the stopping rule returns the first prefix satisfying all pairwise threshold conditions.

The same example also illustrates the risk calibration. Assume equal class priors:
\begin{equation*}
\pi_1=\pi_2=\pi_3=\pi_4=\frac{1}{4}.
\end{equation*}
Since there are six class pairs,
\begin{equation*}
S_\pi=\sum_{1\le i<j\le4}\sqrt{\pi_i\pi_j}=6\sqrt{\frac{1}{4}\frac{1}{4}}=\frac{3}{2}.
\end{equation*}
If the operational threshold is fixed at $\theta=10^{-3}$, the corresponding product-model all-pairs risk level is
\begin{equation*}
\varepsilon(\theta)=\theta S_\pi=10^{-3}\cdot\frac{3}{2}=1.5\times10^{-3}.
\end{equation*}
Consequently, the selected prefix satisfies the product-model bound
\begin{equation*}
R_{\mathrm{AP}}^{\star,\otimes}(T^\star)\le 1.5\times10^{-3}.
\end{equation*}

Conversely, if the target all-pairs risk level is specified first as $\varepsilon=10^{-3}$, the prior-dependent calibration gives
\begin{equation*}
\theta_\varepsilon=\frac{10^{-3}}{3/2}=6.67\times10^{-4}.
\end{equation*}
This threshold is stricter than $10^{-3}$. At $q=4$, some residual overlaps are still above $6.67\times10^{-4}$, so the prefix $T_4$ would not yet be calibrated under this smaller target risk level. The rule would therefore continue further along the ranking, provided later variables supply enough additional residual-overlap reduction.

The example shows how the pairwise quantities can be used diagnostically. The contrasts involving $c_1$ reach the threshold at $q=3$, whereas the contrasts $(c_2,c_3)$, $(c_2,c_4)$ and $(c_3,c_4)$ only reach it at $q=4$. These latter contrasts are the bottlenecks that determine the global stopping point. The output of the method nevertheless remains class-independent: the selected representation is the single subset $T^\star=\{v_1,v_2,v_3,v_4\}$, shared by all classes.

The example has used a common residual-overlap threshold $\theta$ for all class pairs. The framework also permits pair-specific thresholds $\theta_{ij}$, with stopping condition
\begin{equation*}
B_{ij}(T_q)\le \theta_{ij},
\qquad
1\le i<j\le k.
\end{equation*}
This allows stricter tolerances for scientifically, clinically or operationally important contrasts, while assigning less stringent tolerances to other pairs. The product-model all-pairs risk bound is controlled whenever
\begin{equation*}
\sum_{1\le i<j\le k}\sqrt{\pi_i\pi_j}\,\theta_{ij}\le \varepsilon.
\end{equation*}
Thus, the uniform rule $\theta_{ij}=\theta$ is a convenient special case, while the pair-specific version distributes the same risk budget unevenly across class contrasts.

\end{document}